\def\year{2021}\relax
\documentclass[letterpaper]{article} 
\usepackage{aaai21}  
\usepackage{times}  
\usepackage{helvet} 
\usepackage{courier}  
\usepackage[hyphens]{url}  
\usepackage{graphicx} 
\usepackage{natbib}  
\usepackage{caption} 
\title{Physarum Powered Differentiable Linear Programming Layers and Applications}

\author {
        Zihang Meng, \textsuperscript{\rm 1} 
        Sathya N. Ravi, \textsuperscript{\rm 2}
                Vikas Singh \textsuperscript{\rm 1} \\
}
\affiliations {
  \textsuperscript{\rm 1} University of Wisconsin-Madison \\
   \textsuperscript{\rm 2} University of Illinois at Chicago \\
  zihangm@cs.wisc.edu,  sathya@uic.edu,  vsingh@biostat.wisc.edu
}

\usepackage{amsmath,amssymb,amsthm}
\usepackage[inline]{enumitem}
\usepackage[utf8]{inputenc}
\usepackage[english]{babel}
\newtheorem{theorem}{Theorem}
\usepackage{multirow, mathtools,amssymb,amsmath,bbm}
\usepackage[ruled, linesnumbered]{algorithm2e}
\usepackage{paralist}
\usepackage{comment}
\usepackage[table]{xcolor}
\usepackage{subcaption}
\newtheorem{assume}{Assumption}
\newtheorem{remark}{Remark}

\newtheorem{obsn}[theorem]{Observation}
\usepackage{wrapfig,lipsum,booktabs}
\newtheorem*{theorem*}{Theorem}
\usepackage[capitalise,nameinlink]{cleveref}

\begin{document}

\maketitle

\begin{abstract}
  Consider a learning algorithm, which involves an internal call to an optimization routine such as a generalized eigenvalue problem, a cone programming problem or even sorting. Integrating such a method
  as a layer(s) within a trainable deep neural network (DNN) in an efficient
  and numerically stable way is not straightforward -- for instance, 
only recently, strategies have emerged for eigendecomposition and differentiable 
sorting. 
We propose an efficient and differentiable solver for general  linear  
programming problems  which  can be  used in a plug and play manner within DNNs as  a layer. 
Our development is inspired by a fascinating but not widely used
link between dynamics of slime mold (physarum) and optimization schemes such as steepest descent. 
We describe our development and show the use of our solver in a video
segmentation task and meta-learning for few-shot learning.  
We review the existing results and provide a technical analysis describing 
its applicability for our use cases.  
Our solver performs comparably with a customized projected gradient descent method 
on the first task and outperforms the differentiable CVXPY-SCS solver  on  the  second  task. 
Experiments  show  that our  solver converges quickly 
without the need for a feasible initial point. Our proposal is easy to implement
and can easily serve as layers whenever a learning procedure needs a fast approximate solution to a LP,
within a larger network. 
\end{abstract}

\section{Introduction}

Many problems in machine learning can be 
expressed as, or otherwise involve as a sub-routine, the minimization 
of a linear function constrained by a set of linear 
equality and inequality constraints, also known as a Linear Program (LP). 
LPs can be solved efficiently even when the problem sizes 
are large, and industrial strength solvers are readily available.
Over the last twenty 
years, direct applications of LPs in 
machine learning and computer vision include 
image reconstruction \cite{tsuda2004image}, denoising \cite{tavakoli2012image}, deconvolution \cite{ahmed2013blind}
surface reconstruction \cite{grady2008minimal}, graphical models \cite{ravikumar2006quadratic}, scene/view understanding \cite{mauro2014integer}, and numerous others.  
While the use of specialized solvers 
based on combinatorial optimization rather than 
the direct use of a simplex or interior point method has 
been more common in large scale settings (e.g., in vision), 
there are also numerous 
instances where LP duality inspired schemes (such as primal-dual methods) 
have led to competitive and/or more general solution schemes. 

{\bf Are LPs needed in modern learning problems?} Within the last decade, deep neural networks 
have come to dominate many AI problems. 
So, an LP (or other well-studied numerical algorithms/methods) will rarely provide an {\em end-to-end} model
for a practical problem. 
Nonetheless, similar to how various 
linear algebra routines such as eigendecomposition still play a key 
role as a sub-routine in modern learning tasks, 
{\em LP type models are still prevalent} 
in numerous pipelines in machine learning. 
For instance, 
consider a 
representation learner defined by taking our favorite off-the-shelf  
architecture where the representations are used to setup 
the cost for a ``matching'' problem (commonly written as a LP). Then, 
once a matching problem is solved, we route 
that output to pass through downstream layers 
and finally the loss is evaluated. 
Alternatively, consider the case where we must reason about (or group)  a set of 
low-level primitives, via solving an assignment problem, 
to define a higher order semantic construct as is 
often the case in capsule networks \cite{sabour2017dynamic}. 
Or, our architecture involves estimating the Optimal transport 
distance \cite{salimans2018improving,bousquet2017optimal,sanjabi2018convergence} where the cost matrix depends on
the outputs of previous layers in a network. Such a module 
(rather, its approximations) lie at the 
heart of many popular methods for training generative 
adversarial networks (GANs) \cite{arjovsky2017wasserstein}. 
Separately, confidence calibration is becoming 
an important issue in deep learning \cite{guo2017calibration,nixon2019measuring};
 several forms of calibration involve solutions to LPs. 
One approach for dealing 
with such a ``in the loop'' algorithmic procedure \cite{amos2017optnet}
is to treat it as a general two-level optimization.
When the feasible set of the LP is a box/simplex or can be represented using
ratio type functions \cite{Ravi_Venkatesh_Fung_Singh_2020}, it is possible to unroll the
optimization with some careful modifications of existing subroutines such as projections. 
This is not as straightforward in general 
where one must also concurrently perform projections on to 
the feasible set. An ideal solution would be a LP module that could be used anywhere in our architecture: one 
which takes its inputs from the previous layers and feeds into subsequent layers in the network. 

{\bf Contributions: Backpropagation through LP.}
The key difficulty in solving LPs within a deep 
network is efficiently 
minimizing a loss $\ell(\cdot)$ which depends on a 
parameter derived from 
the solution of a LP -- we must backpropagate 
{\em through} the LP solver to update the network weights. This problem is, of course, not unique 
to LPs but has been recently encountered in inserting various optimization 
modules as layers in a neural network, e.g., reverse 
mode differentiation through an ODE solver \cite{chen2018neural}, 
differentiable sorting \cite{mena2018learning} and formulating 
quadratic \cite{amos2017optnet} or cone programs as neural network layers 
\cite{agrawal2019differentiable}. Our inspiration is a beautiful link \cite{straszak2015natural,johannson2012slime}
between dynamics of a 
slime mold (physarum polycephalum) and mathematical 
optimization that 
has not received
attention in deep learning. Exploiting the ideas 
in \cite{straszak2015natural,johannson2012slime} with certain adjustments leads to a ``LP module/layers'' called $\gamma-$AuxPD that can be incorporated 
within various architectures. Specifically, our main result in Thm. \ref{lem:bm} together with the results in \cite{straszak2015natural,johannson2012slime} shows that $\gamma-$AuxPD can solve a much larger class of LPs. Some immediate advantages of $\gamma-$AuxPD include 
\begin{inparaenum}[\bfseries (a)]
\item simple plug-and-play differentiable LP layers; 
\item converges fast; 
\item does not need a feasible solution as an initialization
\item very easy to integrate or implement.
\end{inparaenum}
We demonstrate how these properties provide a practical and easily usable module for solving LPs.

\subsection{Related Works}
The challenge in solving an optimization 
 module {\em within} a deep network often boils down to the specific steps and the end-goal of that module itself. In some cases (unconstrained minimization of simple functions), the update 
 steps can be analytically calculated \cite{dave2019learning,schmidt2014shrinkage}. 
 For more general unconstrained objectives, we must perform 
 unrolled gradient descent during training  \cite{amos2017input,metz2016unrolled,goodfellow2013multi}.  When 
 the optimization involves certain constraints, one must 
 extend the frameworks to use iterative schemes incorporating projection operators, that repeatedly project the solution into a subspace of feasible solutions \cite{zeng2019dmm}. Since such operators are difficult to differentiate in general, it is hard to incorporate them directly outside of special cases. To this end, \cite{amos2017input} dealt with constraints by incorporating them in the Lagrangian and using the KKT conditions.  For combinatorial problems with linear objectives, \cite{vlastelica2019differentiation}
   implemented an efficient backward pass through blackbox implementations of combinatorial solvers and \cite{berthet2020learning} recently
   reported success with end-to-end differentiable learning with blackbox optimization modules.
 In other cases, when there is no associated objective function, 
 some authors have reported some success with using  reparameterizations 
 for homogeneous constraints 
 \cite{frerix2019linear}, adapting 
 Krylov subspace methods \cite{de2017krylov},
 conditional gradient schemes \cite{ravi2019explicitly}
 and so on. 
 
Our goal here is to incorporate an LP as a module within the network, 
and is related in principle to some other works  that incorporate 
optimization routines of different forms within a deep model 
which we briefly review here. 
In \cite{belanger2016structured}, the authors 
proposed a novel structured prediction network by solving an energy minimization problem within the network whereas \cite{mensch2018differentiable} utilized differentiable dynamic programming for structured prediction and attention. 
To stabilize the training of Generative Adversarial Networks (GANs), 
\cite{metz2016unrolled} defined the generator objective with respect to an unrolled optimization of the discriminator. 
Recently, it has been shown that incorporating concepts such as fairness \cite{sattigeri2018fairness} and  verification \cite{liu2019algorithms} within deep networks also requires solving an optimization model internally. 
Closely related to our work is OptNet \cite{amos2017optnet}, 
which showed how to  design a network architecture that 
integrates constrained Quadratic Programming (QP) as a differentiable layers. While the method is not directly designed to
work for linear programs (quadratic term needs to be positive definite), in experiments, one may add a suitable quadratic term as a regularization.
More  recently, \cite{agrawal2019differentiable} introduces a package for differentiable constrained convex programming. Specifically, it utilizes a solver called SCS implemented in CVXPY package \cite{ocpb:16,scs}, which we denote as CVXPY-SCS in our paper.

\section{Why Physarum Dynamics?}
Consider a Linear Program (LP) in the standard form given by, 
\begin{align}
\min_{x\in\mathbb{R}^n} \quad c^Tx \quad  \textrm{s.t.} \quad Ax=b, x\geq 0 \label{lp_standard}
\end{align}
where $A\in \mathbb{R}^{m\times n},c\in \mathbb{R}^{n}_{>0}, b\in \mathbb{R}^m$. In \eqref{lp_standard},  $c$ is called the {\em cost vector} (we explain how to deal with nonpositive $c$ in Section \ref{section_auxiliary}), and the intersection of the linear equalities $Ax=b$, and inequalities $x\geq 0$ is called the  {\em feasible set} denoted by $P$. Now, we briefly discuss two main families of algorithms that are often used to solve LPs of the form \eqref{lp_standard}.

\subsection{Simplex Algorithms: The Workhorse}
Recall that by the Minkowski-Weyl theorem,  the feasible set $P$ can be decomposed into a finite set of extreme points and rays.  A family of algorithms called {\em Simplex} exploits this decomposition of $P$ to solve LPs. 
Intuitively, the Simplex method is based on the principle that if there exists a solution to a LP, then there is at least one vertex (or an extreme point) of $P$ that is optimal. In fact, Simplex algorithms can be seen as {\bf First Order} methods with a careful choice of update direction so as to move along the edges of $P$. There are three key properties of simplex algorithms to solve LP \eqref{lp_standard}: \begin{enumerate*}[label=(\roman*)]
    \item {\em Good:} We can obtain {\em exact} solutions in finite number of iterations;
    \item {\em Bad:} The worst case complexity is exponential in $m$ (or $n$); and 
    \item {\em Highly undesirable:} The update directions are computed by forming the {\em basis} matrix making the algorithm {\em combinatorial/nondifferentiable} in nature.
\end{enumerate*} 
\begin{remark}
It may not be possible to use a differentiable update rule since it would require an enumeration of vertices of $P$ -- exponential in dimensions $n$ \cite{barvinok2013bound}.
\end{remark}

\subsection{Interior Point Algorithms: Trading Exactness for Efficiency}
Asking for {exact} solutions of LP \eqref{lp_standard} may be a stringent requirement. An approximate solution of LP \eqref{lp_standard} can be computed using a different family of methods called {\em Interior Point Method} (IPM) in $ O(\sqrt{\max(m,n})$ \cite{wright1997primal}. Intuitively, while the iterates of a simplex method proceed along the edges of $P$, an IPM passes through the {\em interior} of this polyhedron. In particular, IPMs are {\bf second order} algorithms since they directly solve the system of nonlinear equations derived from KKT conditions by applying variants of Newton's method \cite{wright1997primal}.  As with Simplex methods, we point out to three key properties of IPM: \begin{enumerate*}[label=(\roman*)]
    \item {\em Good:} IPM based algorithms can efficiently solve LP \eqref{lp_standard} in theory \cite{lee2014path,gondzio2012interior};
    \item {\em Bad:} IPMs {\em need} to be started from a feasible point although there are special infeasible start IPMs \cite{roos2006full}; and 
    \item {\em Bad:} In practice, IPMs are faster than Simplex Method {\em only} when $m$, and $n$ are large, e.g., millions \cite{Cui2019}.
\end{enumerate*} \begin{remark}
Even if we can find a feasible point efficiently, it is not easy to {\em warm start} IPM methods due to the high sensitivity of the central path equation \cite{John2008}. In contrast, first order methods like Simplex can be easily warm started \cite{arsham1997classroom}.
\end{remark}

\subsection{Physarum Dynamics: Best of Both Worlds?}
The term {\em Physarum Dynamics (PD)} refers to the movement of a slime mold called {\it Physarum polycephalum},
and is  studied in mathematical biology for its inherent computational nature and properties that closely mirror mathematical optimization. For example, in an interesting result, \cite{toshiyuki2000maze} showed that the slime mold can solve a shortest path problem on a maze. Further, the temporal evolution of Physarum has been used to 
learn robust network design \cite{tero2007mathematical,johannson2012slime}, by connecting it to a broad class of dynamical systems for basic computational problems such as shortest paths and LPs. 
In \cite{straszak2015natural}, the authors studied the convergence properties of PD for LPs, and showed that these steps surprisingly mimic a steepest-descent type algorithm on a certain Riemannian manifold. While these interesting links have not 
been explored in AI/deep learning, we find that the simplicity of these 
dynamics and its mathematical behavior provide 
an excellent approach towards our key goal.



\label{physarum}

We make the following mild assumption about LPs \eqref{lp_standard}  that we consider here \begin{assume}[Feasibility]
The feasible set $P:=\{x:Ax=b,x\geq0\}$ of \eqref{lp_standard} is nonempty.\label{assm:feas}
\end{assume}
For the  applications considered in this paper, Assumption \ref{assm:feas} is always satisfied.  We now describe the PD for solving  LPs and illustrate the similarities and differences between PD and other methods.

Consider any vector $x\in \mathbb{R}^n$ with $x>0$ and let $W\in \mathbb{R}^{n\times n}$ be the diagonal matrix with entries $\frac{x_i}{c_i},i=1,2,...,n$. Let $L=AWA^T$ and  $p\in \mathbb{R}^m$ is the solution to the linear system $Lp=b$. Let $q=WA^Tp$. The PD for a LP (e.g., in \eqref{lp_standard}) given by $(A,b,c)$ is defined as,
\begin{align}
    \frac{dx_i(t)}{dt}=q_i(t)-x_i(t),\quad  i=1,2,\dots,n.
\end{align}
Equivalently, using the definition of $q$ we can write the {\em continuous} time PD compactly as,
\begin{align}
    \dot{x}=W(A^TL^{-1}b-c).\label{eq:contpd}
\end{align}  Theorem 1 and 2 in \cite{straszak2015natural} guarantee that \eqref{eq:contpd} converges to an $\epsilon-$approximate solution efficiently with no extra conditions and its discretization converges as long as the positive step size is small enough.
\begin{remark}[PD vs IPM]
  Similar to IPM, PD requires us to compute a full linear system solve at each iteration. However, note that the matrix $L$ associated with linear system in PD is completely different from the KKT matrix that is used in IPM. Moreover, it turns out that unlike most IPM,  PD can be started with an {\bf infeasible starting point}. Note that PD only requires the initial point to satisfy $As=b$ which corresponds to solving ordinary least squares  which can be easily done using any iterative method like Gradient Descent.
\end{remark}
\begin{remark}[PD vs Simplex] Similar to Simplex,  PD corresponds to a gradient, and therefore is a {\em first order} method. The crucial difference between the two methods, is that the metric used in  PD is {\bf geodesic} whereas Simplex uses the Euclidean metric. Intuitively, using the geodesic metric of $P$ instead of the Euclidean metric can vastly improve the convergence speed since the performance of first order methods is dependent on the choice of coordinate system \cite{yang1998efficiency,zhang2016first}.
\end{remark}
 {\bf When is PD efficient?} 
 As we will see shortly in Section \ref{Applications}, in the two applications that we consider in this paper, the sub-determinant of $A$ is {\em provably} small -- constant or {at most quadratic in $m,n$}. In fact, when $A$ is a node incidence matrix, PD computes the shortest path, and is known to converge extremely fast. In order to be able to use PD for a wider range of problems, we propose
 a simple modification described below. Note that since many of the vision problems
 require auxiliary/slack variables in their LP (re)formulation, the convergence results in \cite{straszak2015natural} do
 {\em not} directly apply since $L$ in \eqref{eq:contpd} is {\em not} invertible. Next,
 we discuss how to deal with noninvertibility of $L$ using our proposed algorithm called $\gamma-$AuxPD (in Algorithm \ref{alg}). 
 
\section{Dealing with Auxiliary Variables using $\gamma-$AuxPD}
\label{section_auxiliary}
In the above description, we assume that $c\in\mathbb{R}^n_{>0}$. We now address the case where $c_i=0$ under the following assumption on the feasible set $P$ of LP \eqref{lp_standard}: \begin{assume}[Bounded]
The feasible set $P\subseteq [0,M]^n$, i.e., $x\in P \implies x_i\leq M~\forall~i\in[n].$\label{assm:bdd}
\end{assume} Intuitively, if $P$ is bounded, we may expect that the optimal solution set to be invariant under a sufficiently small perturbation of the cost vector along {\em any} direction.  The following observation from \cite{johannson2012slime} shows that this is indeed possible as long as $P$ is finitely generated:
\begin{obsn}[{{\cite{johannson2012slime}}}]\label{obsn1}
Let $\epsilon>0$ be the given desired level of accuracy, and say  $c_i=0$ for some $i\in [n]$. Recall that our goal is to find a point $\hat{x}\in P$ such that $c^T\hat{x}-c^Tx^*\leq \epsilon$ where $x^*$ is the optimal solution to the LP \eqref{lp_standard}. Consider the $\gamma-$perturbed LP given by $\{A,b,\hat{c}\}$, where $\hat{c}_i=c_i$ if $c_i>0$ and $\hat{c}=\gamma$ if $c_i=0$. Let $x_2$ be an extreme point that achieves the second lowest cost to LP \eqref{lp_standard}. Now it is easy to see that if  $\gamma<\frac{\delta}{n\cdot M}$ where $\delta=c^Tx_2-c^Tx^*$,  then $x^*$ is  an approximate solution of $\{A,b,\hat{c}\}$. Hence, it suffices to solve the $\gamma-$perturbed LP. 
\end{obsn}
With these modifications, we present our discretized $\gamma-$AuxPD algorithm \ref{alg} that solves a slightly perturbed version of the given LP. 
\begin{remark}
     Note that $\gamma-$perturbation argument does not work for any $P$ and $c$ since LP \eqref{lp_standard} may be unbounded or have {\em no} extreme points. 
\end{remark}
Observation \ref{obsn1} can be readily used for computational purposes by performing a binary search  over $\gamma$
if we can obtain a finite upper bound $\gamma_u$. Furthermore, if $\gamma_u$ is a polynomial function of the input parameters $m,n$ of LP, then Observation \ref{obsn1} implies that $\gamma-$AuxPD algorithm is also efficient.    Fortunately, for applications that satisfy the bounded assumption \ref{assm:bdd}, our Theorem \ref{lem:bm} shows that a {\em tight} upper bound $\gamma_u$ on $\gamma_P$ can be provided in terms of $M$ (diameter of $P$).

{\bf Implementation.} Under Assumption \ref{assm:bdd}, negative costs can be handled by replacing $x_i=-y_i$ whenever $c_i<0$, or in other words, by flipping the coordinate axis of coordinates with negative costs, which has been noticed in \cite{johannson2012slime}.
Since we use an iterative linear system solver to compute $q$, we project $x$ on to $\mathbb{R}_{\geq\epsilon}$
after each iteration: this corresponds to a simple clamping operation.

\begin{algorithm}[!bt]
 \SetAlgoLined
 \textbf{Input:} LP problem parameters $A,b,c$, initial point $x_0$, Max iteration number $K$, step size $h$, accuracy level  $\epsilon$, approximate diameter $\gamma_P$   \\
 Set $x_s \leftarrow x_0$ if $x_0$ is provided else $\text{rand}([n],(0,1))$\\
 Perturb cost $c \leftarrow c + \gamma_P\bf{1}_0$ where $\bf{1}_0$ is the binary vector with unit entry on the indices $i$ with $c_i=0$  \\
\For{$i=1$ to $K$}{
 Set: $W \leftarrow \text{diag}(x_s/c)$\\
 Compute: $L \leftarrow AWA^T$\\
 Compute: $p \leftarrow L^{-1}b$ using iterative solvers\\
 Set: $q \leftarrow WA^Tp$\\
 Update: $x_s \leftarrow (1-h)x_s + h q$\\
 Project onto $ \mathbb{R}_{\geq\epsilon}$: $x_s \leftarrow  \max\left(x_s,\epsilon\right)$
 } 
 \textbf{Return:} $x_s$ 
 \caption{$\gamma-$AuxPD Layer}
 \label{alg}
\end{algorithm}

\section{Analysis of Some Testbeds for $\gamma-$AuxPD: Bipartite Matching and SVMs}
In order to illustrate the potential of the $\gamma-$AuxPD layer (Alg. \ref{alg}), we consider two classes of LPs 
common in
a number of applications and show that they can be solved using $\gamma-$AuxPD. These two 
classes of LPs are chosen because they link nicely to interesting problems
involving deep neural networks which we study in \S\ref{Applications}.

\begin{table}[b]

\centering
\resizebox{1.0\columnwidth}{!}{
\smallskip
\begin{tabular}{l  | c c c |c c c }
\hline
 & \multicolumn{3}{c}{\bf $\gamma-$AuxPD}& \multicolumn{3}{c}{\bf PGD-Dykstra}  \\
\hline
  Iter. \#   & $10$  &  $50$  &  $100$ & $10$   & $50$     & $100$   \\
 Proj. \#   & $NA$  &  $NA$  &  $NA$ & $5$   & $10$     & $50$  \\
Objective   & $0.100$  &  $0.098$  &  $0.099$ & $0.137$   & $0.121$     & $0.120$  \\
Time (s)   & $0.016$  &  $0.040$  &  $0.071$ & $0.016$   & $0.146$     & $0.498$   \\
\hline
\end{tabular}
}
\caption{Results on solving random matching problems.}
\label{table_PGD}
\end{table}

\subsection{Bipartite Matching using Physarum Dynamics}
\label{examples_matching}
Given two finite non-intersecting sets $I$, $J$ such that $ |I|=m$,  $|J|=n,n\ll m $, and a cost function $C:I\times J\to\mathbb{R}$, solving a minimum cost bipartite matching problem corresponds to finding a map $f:I\to J$ such that total cost  $\sum_iC(i,f(i))$ is minimized. If we represent $f$ using an assignment matrix $X\in \mathbb{R}^{n\times m}$, then a LP relaxation of the matching problem  can be written in  standard form \eqref{lp_standard} as,
\begin{align}
\min_{(X,s_m)\geq 0}~  & ~\mbox{tr}(CX^T) + \gamma{\bf 1}_m^Ts_m \nonumber \\
\textrm{s.t.}~ & ~X\boldsymbol{1}_m=1_n, ~ X^T{1_n} + s_m = {1_m} \label{eq:mgreatern}
\end{align}
where $C\in \mathbb{R}^{n\times m}$ is the cost matrix, $\boldsymbol{1}_d$ is the all-one vector in $d$ dimension, and $s_m\in R^m$ is the slack variable. 
\begin{remark}
Note that in LP \eqref{eq:mgreatern}, the slack variables $s_m$ impose the $m$ inequalities  $X^T1_n\leq 1_m$.
\end{remark}
The following theorem shows that the convergence rate of
$\gamma-$AuxPD applied to the bipartite matching in \eqref{eq:mgreatern} only has a dependence which is {\bf logarithmic} in $n$.
\begin{theorem}\label{lem:bm}
Assume we set $0<\gamma\leq \gamma_u$ such that $1/\gamma_u=\Theta(\sqrt{m})$. Then, our $\gamma-$AuxPD (Algorithm \ref{alg}) converges to an optimal solution to \eqref{eq:mgreatern} in $\Tilde{O}\left(\frac{m}{\epsilon^2}\right)$ iterations where $\Tilde{\mathcal{O}}$ hides the logarithmic factors in $m$ and $n$.
\end{theorem}
\begin{proof}(Sketch) To prove Theorem \ref{lem:bm}, we use a result from convex analysis called the {\em sticky face lemma} to show that for all small perturbations of $c$, the optimal solution set remains invariant. We can then simply estimate $\gamma_u$
 to be the largest acceptable perturbation (which may depend on $C,P$ but {\em not} on any combinatorial function of $P$ like extreme points/vertices). See Section \ref{app:thm2_pf} for details.
 \end{proof}

{\bf Verifying Theorem \ref{lem:bm}.} We construct random matching problems of size $n=5,m=50$ (used later in \S\ref{DMM})  with batch size $32$, where we randomly set elements of $C$ to be values in $[0,1]$. 
We compare our method with CVXPY-SCS and  a projected gradient descent algorithm in which the projection exploits the Dykstra’s algorithm (used by \cite{zeng2019dmm} in \S\ref{DMM}) (we denote it as PGD-Dykstra).

{\bf Evaluation Details.} We run $100$ random instances of matching problems for both our $\gamma-$AuxPD algorithm and PGD-Dykstra with different number of iterations. 
We report the objective value computed using the solution given by our $\gamma-$AuxPD solver/PGD-Dykstra/CVXPY-SCS. 
Our step size is $1$ and learning rate of PGD-Dykstra is set to $0.1$ (both used in \S\ref{DMM_exp}). For CVXPY-SCS, the number of iterations is determined by the solver itself for each problem and it gets $0.112$ objective with mean time $0.195$ (s). The results of $\gamma-$AuxPD and PGD-Dykstra are reported in Table \ref{table_PGD}.  Our $\gamma-$AuxPD algorithm achieves faster convergence and
better quality solutions.

\subsection{$\ell_1$-normalized Linear SVM using $\gamma-$AuxPD}
In the next testbed for $\gamma-$AuxPD, we solve a $\ell_1$-normalized linear SVM \cite{hess2015support} in the standard form of LP \eqref{lp_standard}. Below, $\widetilde{K}^{[i,j]}$ stands for  $K(x_i,x_j)(\alpha_{1j}-\alpha_{2j})$:

\begin{equation}
\label{l1_svm}
\begin{aligned}
&\min_{\alpha_1,\alpha_2,s,b_1,b_2,\xi} \quad  \sum_{i=1}^ns_i + C\sum_{i=1}^n(\xi_i + 2z_i) \\
& \textrm{s.t.} \quad  y_i\left(\sum_{j=1}^{n}y_j \widetilde{K}^{[i,j]} +(b_1-b_2)\right)+\xi_i-Mz_i-l_i =  1,\\
& \sum_{j=1}^ny_j \widetilde{K}^{[i,j]} - s_i + p_i = 0, \quad
\sum_{j=1}^ny_j \widetilde{K}^{[i,j]} + s_i - q_i = 0,\\
&z_i+r_i=1, \quad 
\alpha_1,\alpha_2,s,b_1,b_2,\xi, z_i, l_i,p_i,q_i,r_i, \geq 0 \\
& \forall i=1,2,\cdots,n.
\end{aligned}
\end{equation}

Like Thm. \ref{lem:bm}, we can show a convergence result for $\ell_1$-SVM \eqref{l1_svm} (see Section \ref{app:l1_svm}). 

{\bf Verifying convergence of $\gamma-$AuxPD for $\ell_1$-SVM \eqref{l1_svm}.} We compare our method with the recent CVXPY-SCS solver \cite{agrawal2019differentiable} which can also solve LPs in a differentiable way. We constructed some simple examples to check whether CVXPY-SCS and our $\gamma-$AuxPD solver works for SVMs (e.g., binary classification where training samples of different class come from Gaussian distribution with different mean). Both $\gamma-$AuxPD and CVXPY-SCS give correct classification results. We will further show in \S\ref{Meta} that when used in training, $\gamma-$AuxPD achieves better performance and faster training time than CVXPY-SCS.

\section{Differentiable LPs in Computer Vision}
\label{Applications}

\begin{figure*}[!t]
\centering
\includegraphics[width=0.7\textwidth]{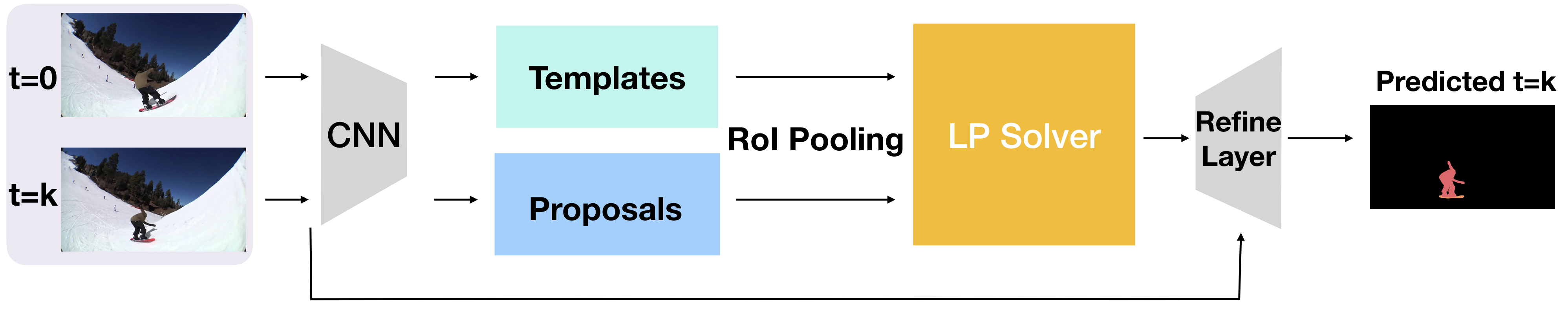} 
\caption{\small Architecture of DMM \cite{zeng2019dmm}: The yellow box is where the linear program is solved. In this application the linear program is a bipartite matching problem. }
\label{fig_dmm}
\end{figure*}

\begin{figure*}[!t]
\centering
\includegraphics[width=0.7\textwidth]{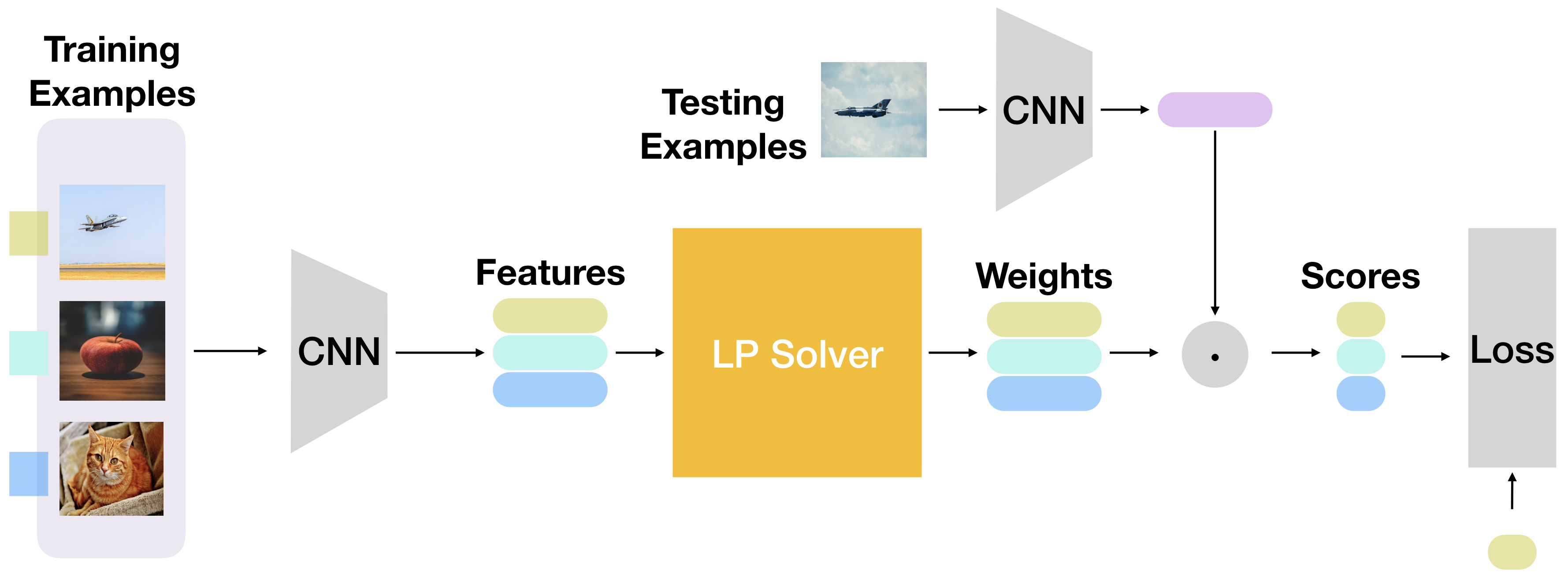} 
\caption{\small Architecture of Meta-learning \cite{lee2019meta}: The yellow box is where the linear program is solved. In this application, the linear program is a linear SVM. }
\label{fig_meta}
\end{figure*}

We now demonstrate the versatility of our $\gamma-$AuxPD layer 
in particular scenarios in computer vision. Our goal here is to 
show that while the proposed procedure is easy, 
it can indeed be used in a plug and play manner in fairly 
different settings, where the current alternative 
is either to design, implement and debug a specialized 
sub-routine \cite{zeng2019dmm} or to utilize 
more general-purpose schemes when a simpler one would suffice 
(solving a QP instead of a LP) as in \cite{lee2019meta}. We try to 
keep the changes/modifications to the original pipeline where 
our LP solver is deployed as minimal as possible, so ideally, 
we should expect  that there are no 
major fluctuations in the overall accuracy profile.  


\subsection{Differentiable Mask-Matching in Videos}
\label{DMM}

We review the key task from \cite{zeng2019dmm} to introduce the differentiable mask-matching network for video object segmentation, and how/why it involves a LP solution. The overall architecture is in Fig. \ref{fig_dmm}.

{\bf Problem Formulation.}
Given a video with $T$ frames as well as the mask templates in the first frame, the goal is to obtain a segmentation of the same set of instances in all of remaining frames. \cite{zeng2019dmm} shows that differentiable matching between the templates and the bounding boxes proposed by the detector achieves superior performance over previous methods.

{\bf LP instance.} 
The goal is to use the cost matrix and solve a matching problem. 
Recall that 
minimum-cost bipartite matching can be formulated as a integer linear program (ILP) and can be relaxed to a LP, given by the formulation in standard form stated in \eqref{eq:mgreatern} (identical to the ILP and LP in 
\cite{zeng2019dmm}). 
The number of proposals $m$ is much larger than the number of templates $n$ and so one would ask that $X^T\boldsymbol{1_n}\leq \boldsymbol{1}_m$ instead of $X^T\boldsymbol{1_n} = \boldsymbol{1}_m$. 
%
%
 
{\bf Solver.} In \cite{zeng2019dmm}, the authors use a specialized 
projected gradient descent algorithm with a cyclic constraint projection 
method (known as \textit{Dykstra's algorithm}) to solve the LP. The constraints 
in this LP are simple enough that calculating the projections is not complicated 
although the convergence rate is {\bf not known}.
We can directly replace their solver  with $\gamma-$AuxPD in Alg. \ref{alg} to solve 
the problem, also in a differentiable way. Once the solution is obtained, 
\cite{zeng2019dmm} uses a mask refinement module which we also use to 
ensure consistency between the pipelines. 

\subsubsection{Experiments on Youtube-VOS.}
\label{DMM_exp}





{\bf Parameter settings.}
The projection gradient descent solver in \cite{zeng2019dmm} has three parameters to tune: number of gradient steps, number of projections, and learning rate. We use $N_{grad}=40,N_{proj}=5,lr=0.1$ as in their paper to reproduce their results. For $\gamma-$AuxPD layer, the choice is simple: step size $h=1$ and $K=10$ iterations work well for both two experiments and the other tests we performed.
 From Table \ref{table_PGD} we can see that the PGD-Dykstra solver from \cite{zeng2019dmm} is faster and more tailormade for this application than CVXPY-SCS thus we only compare with the PGD-Dykstra solver for this application. 

\begin{table}[!bt]
{

\resizebox{.99\columnwidth}{!}{
\centering
\begin{tabular}{l  | c c c c c c}
\hline
 & $\mathcal{J}_m$  &$\mathcal{J}_r$  &$\mathcal{J}_d$  &$\mathcal{F}_m$  &$\mathcal{F}_r$ &$\mathcal{F}_d$  \\
\hline
DMM-Net \cite{zeng2019dmm}    & $63.4$  &  $72.7$  &  $9.3$ & $77.3$   & $84.9$     & $10.5$     \\
$\gamma-$AuxPD layer   & $63.4$  &  $72.2$  &  $9.2$ & $77.3$   & $85.3$     & $10.4$    \\
\hline

\end{tabular}
}
\caption{\label{table_youtube}Results on Youtube-VOS train-val split. Subscripts $m,r,d$ stand for mean, recall, and decay respectively.}
}

\end{table}

\textbf{How do different solvers compare on Youtube-VOS?}
Our final results are shown in Table \ref{table_youtube}. Our solver works well 
and since the workflow is near identical to \cite{zeng2019dmm}, we achieve comparable results with \cite{zeng2019dmm} while achieving small benefits in inference time. We notice that although our solver performs better
for a simulated matching problems; 
since the matching problem here is small and the cost matrix learned by the feature extractor is already good (so easy to solve), the runtime behavior is similar. 
Nonetheless, it shows that the general-purpose solver can be directly plugged in and offers 
performance which is as good as a {\em specialized solution} in \cite{zeng2019dmm} that exploits the properties 
of the particular constraint set.


\subsection{Meta-learning for Few-shot Learning}
\label{Meta}
We  briefly  review  the  key  task  from  \cite{lee2019meta}  to  introduce the few-shot learning task using a meta-learning approach, and how it involves getting a solution to a LP. Due to
limited space, we refer readers to \cite{lee2019meta} for more  details of the meta-learning for few-shot learning task. 
 The overall architecture is in Fig. \ref{fig_meta}.

\textbf{Problem Formulation.}
Given a training set $D^{train}=\{(\boldsymbol{x}_t,y_t)\}_{t=1}^T$, in this problem, 
the goal of the base learner $\mathcal{A}$ is to estimate parameters $\theta$ of the predictor $y=f(\boldsymbol{x};\theta)$ so that it generalizes well to the unseen test set $D^{test} = \{(\boldsymbol{x}_t,y_t)\}_{t=1}^Q$. The meta learner seeks to learn an embedding model $\phi$ that minimizes the generalization error across tasks given a base learner $\mathcal{A}$.

\textbf{LP instance.}
There are several requirements for the base learners. First, the evaluation needs to be very efficient since a base learner needs to be solved in every iteration within the \textit{meta-learning} procedure. Second, we need to be able to estimate 
and backpropagate the gradient from the solution of the base learner back to the embedding model $f_\phi$, which means that the solver for the base learner needs to be differentiable. 
In \cite{lee2019meta}, the authors use a multi-class linear support vector machine (SVM) with an $\ell_2$ norm on 
the weights \cite{crammer2001algorithmic}. 
Instead, to instantiate an LP, we use a $\ell_1$ normalized SVM proposed by \cite{hess2015support}. The optimization model for this SVM in a standard form is shown in \eqref{l1_svm}.
%
%
This is a binary SVM model, on top of which we run 
$k \choose 2$ pairwise SVMs to obtain the solution 
where $k$ is the number of classes in the task. 

{\bf Solver.} In \cite{lee2019meta}, the authors use OptNet. 
Note that the number of parameters is only related to the number of training examples and the number of classes, which is often much smaller than the dimensionality of the features for few-shot learning. Since feature selection 
seems more appropriate here, 
we may directly replace OptNet with our $\gamma-$AuxPD layer to solve the $\ell_1$-SVM efficiently. Our baseline method
is CVXPY-SCS \cite{agrawal2019differentiable}. The implementation of Optnet \cite{amos2017optnet} does not directly support solving LPs since it requires a positive definite quadratic term. Still, to
test its ability of solving LPs, we add a diagonal matrix with a small value ($0.1$,
since diagonal value smaller than $0.1$ leads to numerical errors in our experiment) as
the quadratic term (can be thought of as a regularization term). 

\begin{table*}[t]

\centering
\smallskip
\begin{tabular}{l  c c c c}
\hline
 &  \multicolumn{2}{c}{ CIFAR-FS 5-way} & \multicolumn{2}{c}{ FC100 5-way}\\
  LP Solver                        & 1-shot     & 5-shot     & 1-shot & 5-shot\\
\hline

MetaOptNet-CVXPY-SCS             & $70.2\pm 0.7 $  &  $83.6 \pm 0.5 $  &  $ 38.1 \pm 0.6 $   & $51.7 \pm 0.6$  \\
MetaOptNet-Optnet (with regularization) & $69.9\pm 0.7$ & $83.9\pm0.5$ & $37.3\pm 0.5$ & $52.2\pm 0.5$ \\
MetaOptNet-$\gamma-$AuxPD (Ours)             & $71.4 \pm 0.7 $  &  $84.3 \pm 0.5 $  &  $38.2 \pm 0.5 $   & $54.2 \pm 0.5$  \\
\hline$ $
\end{tabular}
\caption{Results on CIFAR-FS and FC100. In $K$-way, $N$-shot few shot learning,  $K$ is the number of classes and $N$ is the number of training examples per class. Performance of more baseline methods is in appendix Table \ref{table_cifar_app}. }
\label{table_cifar}
\end{table*}

\subsubsection{Experiments on CIFAR-FS and FC100.}

\textbf{Datasets.}
We follow the code from \cite{lee2019meta} to conduct the experiments on {\bf CIFAR-FS} and {\bf FC100}. Other training details and dataset information are in the supplement.

\textbf{How do different solvers compare on CIFAR-FS and FC100?}
The results on CIFAR-FS and FC100 are shown in Table \ref{table_cifar}. Using the $\ell_1$ normalized SVM, our solver achieves better performance than CVXPY-SCS \cite{agrawal2019differentiable} and Optnet (with a small quadratic term as regularization) on both datasets and both the $1$-shot and $5$-shot setting. Expectedly, since the pipeline 
is very similar to \cite{lee2019meta}, we achieve a similar 
performance as reported there, although their results 
were obtained through a different solver. 
This suggests that our simpler solver works at least as well, 
and no other modifications were needed. 
Importantly, during the training phase, our solver achieves {\bf $4\times$ improvement in runtime} compared with CVXPY-SCS (baseline which can also solve the $\ell_1$-SVM). 
\cite{lee2019meta} also reported the performance 
of solving $\ell_2$ normalized SVM. The choice of $\ell_1$ versus $\ell_2$ often depends on specific application settings.

We also compare the time spent on solving a batch of LP problems with $n=92,m=40,p=122$ (same size used in the experiment), where $n$ is number of variables, $m$ is number of equality constraints and $p$ is the number of inequality constraints in the original problem form. 
Table \ref{runtime_exp} shows that our implementation is efficient for batch processing on GPU, which is crucial for many modern AI applications.  We also performed a GPU memory consumption comparison with a batch size of $32$: our solver needs
  $913$MB GPU memory, CVXPY-SCS needs $813$MB and Optnet needs $935$MB which are mostly comparable. 

\begin{table}[!bt]
\center
\begin{tabular}{l  | c c c c}
\hline
 batch size & $8$  &$32$  &$128$   \\
\hline
  CVXPY-SCS  & $32.3$  &  $122.7$  &  $455.2$   \\
\hline
  Optnet  & $42.4$  &  $88.1$  &  $243.7$   \\
\hline
  $\gamma-$AuxPD (Ours)   & $24.0$  &  $25.1$  &  $25.8$   \\
\hline
\end{tabular}
\caption{Time (ms) spent on solving a batch of LP problems. The time reported here for CVXPY-SCS does not include that spent on constructing the canonicalization mapping.}
\label{runtime_exp}
\end{table}


\textbf{How does LP solver influence the global convergence of the task?}
To understand how the quality of LP solver influences the global convergence of the learning task
(i.e., where the LP is being used), we conduct a simple experiment. This addresses the
question of whether a good LP solver is really needed?
Here, we add a random Gaussian noise with zero mean and small variance to the
solution of LP solver (to emulate results from a worse solver)
and observe the convergence and final accuracy in the context of the task.
We can see in Table \ref{noise_exp} that the quality of LP solution has a clear influence on the overall performance of the training (few-shot learning in this example).
\begin{table}[!bt]

\resizebox{.99\columnwidth}{!}{
\centering
\begin{tabular}{l  | c c c c c c}
\hline
 Variance of  noise & $0$  &$0.01$  &$0.03$  &$0.05$  &$0.1$ \\
\hline
Test accuracy    & $71.4$  &  $70.1$  &  $69.1$ & $68.2$   & $61.91$   \\
\hline
\end{tabular}

}
\caption{Experiment on CIFAR-FS 5-way 1-shot setting where zero mean random Gaussian noise is added to the solution of $\gamma-$AuxPD solver.}
\label{noise_exp}
\end{table}

\section{Discussion}
\subsection{Other Potential Applications}
\label{more_app}
Linear programming appears frequently in machine learning/vision, and $\gamma-$AuxPD
can be potentially applied fairly directly. We cover a few recent examples which are interesting since they are not often solved as a LP.  
\medskip

{\bf Differentiable Calibration.}
Confidence calibration is important for many applications, e.g., self-driving cars \cite{bojarski2016end} and medical diagnosis \cite{liang2020improved}. However, it is
known that SVMs and deep neural networks give a poor estimate of the confidence to their outputs.
In general, calibration is used only as a post-procedure \cite{guo2017calibration}. Observe that some calibration methods can be written or relaxed in the form of a LP.
For example, Isotonic regression \cite{guo2017calibration},
fits a piecewise non-decreasing function to transform uncalibrated outputs.
By using a $\ell_1$ loss, Isotonic regression can be written as a linear program.
Therefore $\gamma-$AuxPD layer can solve it differentiably within an end to end network
during training, which may be a desirable and lead to better calibration.
\medskip

{\bf Differentiable Calculation of Wasserstein Distance (WD).}
WD is widely used in generative adversarial models \cite{arjovsky2017wasserstein} as well as
the analysis of shapes/point clouds \cite{trillos2017gromov}.
An entropy regularized LP formulation of WD can be solved using the Sinkhorn algorithm.
Results in \cite{amari2019information} suggest that Sinkhorn may be suboptimal since the
limit of the sequence generated by the Sinkhorn algorithm may not coincide with the minimizer
of the unregularized WD. Interestingly, we can apply Thm \ref{lem:bm} (or Thm. 1 in \cite{straszak2015natural})
to conclude that PD (i) is asymptotically {\em exact}; and (ii) matches  the convergence rate of
the Sinkhorn algorithm. For training deep networks, this means that we can 
obtain unbiased gradients using $\gamma-$AuxPD layers which may lead to faster training. 
 \medskip

{\bf Differentiable Hierarchical Clustering.} Hierarchical clustering algorithms are often used in segmentation based vision tasks, see \cite{arbelaez2010contour}. It is well known that an approximate hierarchical clustering can be computed by first rounding the optimal solution of  a LP relaxation, see 
\cite{charikar2017approximate}. Observe that the LP formulation of the sparsest cut problem has more constraints than decision variables owing to the ultrametric requirement of the decision variables. Hence, $\gamma-$AuxPD may be employed to approximately solve the Hierarchical Clustering problem, thus enabling us to differentiate through clustering based objective functions in end-to-end deep learning frameworks. Until recently, the EM-style clustering was a bottleneck 
\subsection{Implicit Differentiation of PD} \label{app:imp_pd}
In this section we show how to get implicit differentiation of $c$. Then $A$ and $b$ follow similarly. Let the updating direction $P(x)=W(A^{T}L^{-1}b-c)$, where $W$ is the diagonal matrix with entries $\frac{x_i}{c_i}$, denoted as $\text{diag}(x\oslash c)$, and $L=AWA^{T}$, where $\oslash$ is element-wise division. So we can rewrite $P$ at the optimal solution $x^{*}$ as 
\begin{equation}
	P(x^*)=\text{diag}(x^*\oslash c)(A^{T}(A\cdot \text{diag}(x^*\oslash c)A^{T})^{-1}b-c)
\label{P_func}
\end{equation}

As a joint function of $c,x^*$, we differentiate both the sides of \eqref{P_func} with respect to $c$ as,\begin{align}
\frac{\partial P}{\partial c} + \frac{\partial P}{\partial x^*} \frac{\partial x^*}{\partial c} =0 \implies  \frac{\partial x^*}{\partial c} = -\left(\frac{\partial P}{\partial x^*} \right)^{-1} \frac{\partial P}{\partial c}.\label{eq:imp_der}
\end{align}

Denote $t_0=x\oslash c$, $T_1=(A\text{diag}(t_0)A^T)^{-1}$ and $t_2=c\odot c$. $\frac{\partial P}{\partial x^*}$ and $\frac{\partial P}{\partial c}$ can be computed analytically as follows:
\begin{align*}
	&\frac{\partial P}{\partial x^*} = \text{diag}((A^{T} T_1b -c)\oslash c) - \\
	&\text{diag}(t_0)A^TT_1 A \text{diag}(b^T(A \text{diag}(t_0) A^T)^{-1}A \oslash c^T)
\end{align*}
	and
\begin{align*}
	&\frac{\partial P}{\partial c} = \text{diag}(t_0) A^T T_1 A \text{diag}(x^{*T} \odot (b^T(A \text{diag}(t_0)A^T)^{-1}A)\\
	& \oslash t_2^T)  - \text{diag}(t_0) - \text{diag}(x\odot (A^TT_1 b -c)\oslash t_2)
\end{align*}


For computational purposes, without loss of generality,  we can assume that the norm of the gradient to be some fixed value, say one. This is because, for training networks with
PD layer using first order methods, scale or magnitude of the gradients can simply  absorbed in the learning rate (interpreted as a hyperparameter)
and tuned using statistical techniques such as cross validation. Hence, in order to evaluate the quality of the implicit gradient calculated from the above equations, we ignore the scale and use similarity based measures.
To this end, we used our bipartite matching problem as a testbed and
compared explicit gradient (calculated by unrolling the update rules) and
implicit gradient (using the formula above).
A high cosine value between explicit gradient and implicit gradient indicates that
the two gradients are mostly in the same direction.
After running $100$ matching problems with different random cost matrices,
we find that in all cases (as shown in Fig. \ref{gradient_pca}), the $P(x)$ becomes very small (with norm less than $0.01$)
which means that the quality of the final solution from our solver is good.
But we find that in a fair number of cases, the cosine
values are less than $0.99$. We suspect that this is due to the inverse operation in \eqref{eq:imp_der} -- note that both the terms in \eqref{eq:imp_der} are matrices, so after the forward pass, 
we have to solve $n$ linear systems in order to compute the gradient. Indeed, this can be tricky in practice -- the
local geometry around the optimal solution $x^*$ may not be ideal (for example, Hessian at $x^*$ with a high condition number) which can then introduce floating point errors that can affect overall gradient computation significantly. 
Fortunately, we find that our algorithm converges in less than 10 iterations in our experiments, so it is extremely 
convenient to do unrolled gradient computation which tends to perform better with the overall training of the network. 
The above discussion also provides reasons why in Table 3, our
solver performs slightly better than CVXPY and Optnet: both of which are
based on implicit gradients.


\begin{figure}[t]
\centering
\begin{subfigure}[t]{.49\columnwidth}
  \centering
  \includegraphics[width=.95\linewidth]{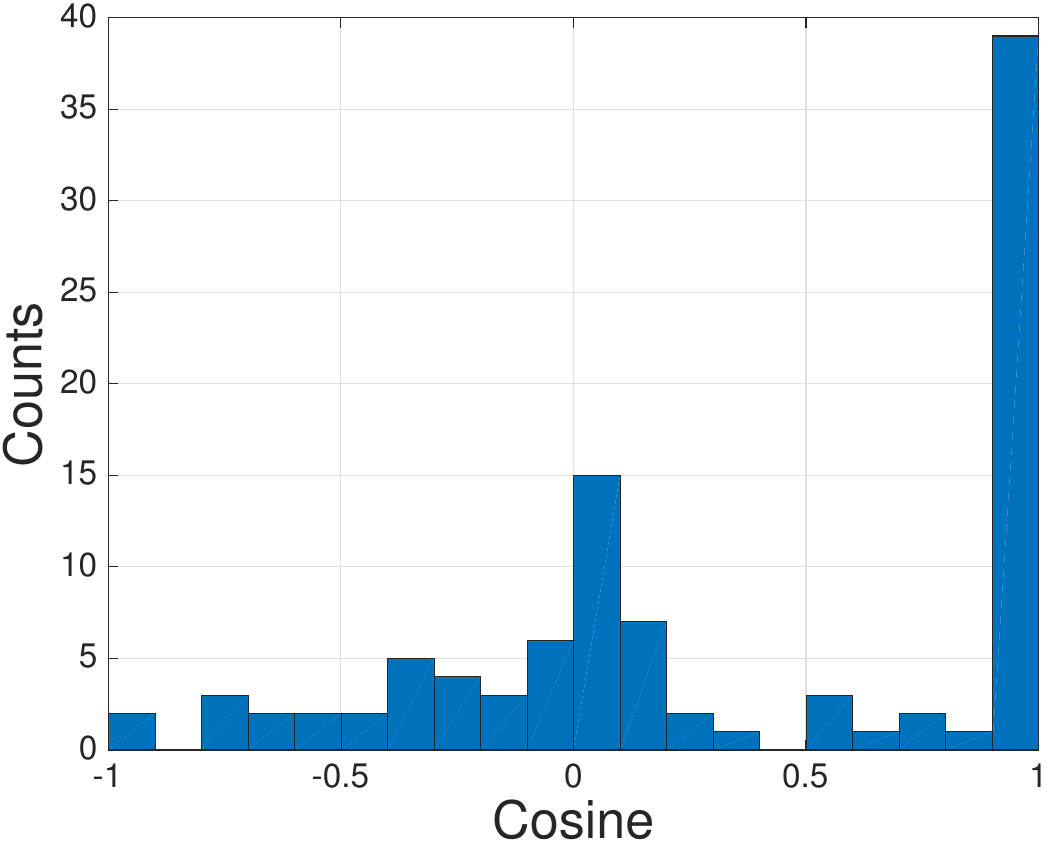}
  \caption{Histogram of the cosine values between implicit and explicit gradient on $100$ random constructed matching problems.}
  \label{gradient_pca}
\end{subfigure}
\begin{subfigure}[t]{.49\columnwidth}
  \centering
  \includegraphics[width=.95\linewidth]{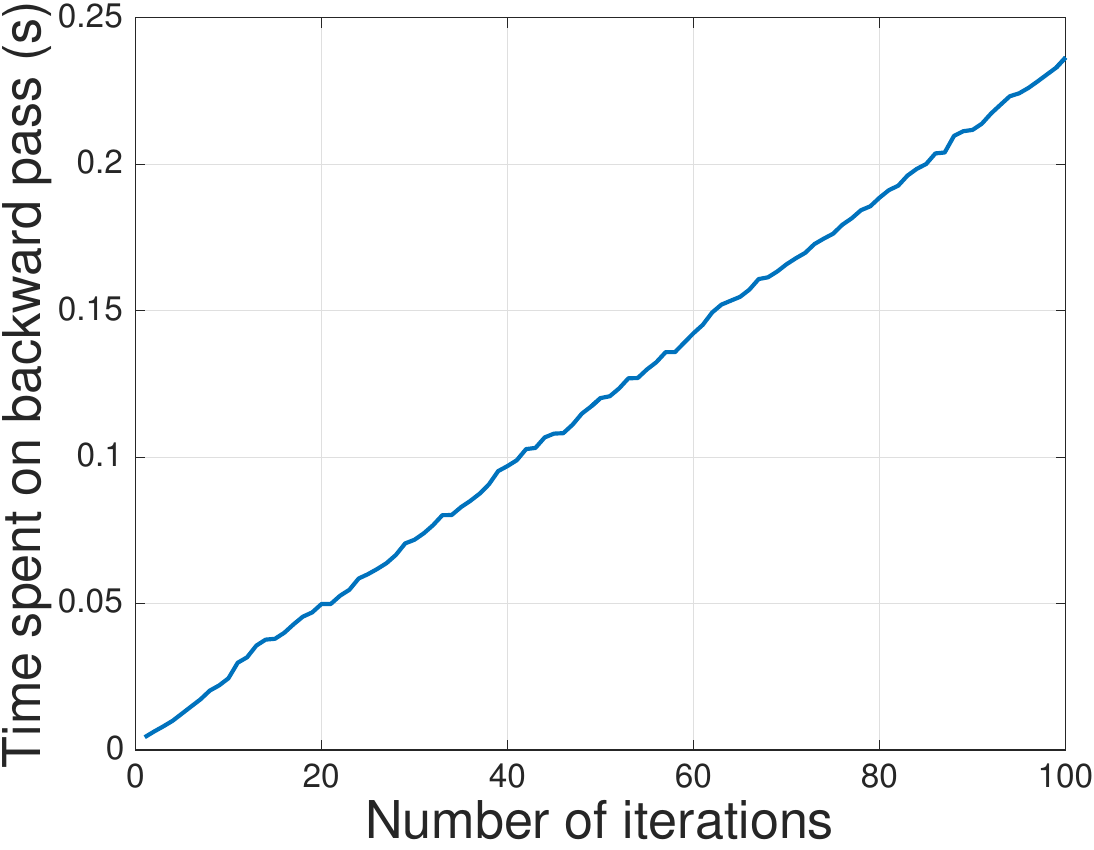}
  \caption{Time cost of explicit differentiation. The time cost increases linearly with the number of iterations.}
  \label{backward_time}
\end{subfigure}
\caption{Comparison between explicit gradient and implicit gradient.}
\label{fig_comparison_implicit_explicit}
\end{figure}



Finally, we note that implicit differentiation may have potential benefits in certain scenarios.
One benefit of implicit differentiation is that the time spent on the backward pass
(gradient backpropagation) is not related to the number of iterations that
our Physarum solver uses in the forward pass.
From this perspective, when is implicit differentiation preferable compared with explicit differentiation (unrolling)? Consider the bipartite matching problem ($m=10,n=50$) as a LP example, we plot the time spent on explicit differentiation as a function of the number of iterations that our Physarum solver uses in forward pass
time in Fig. \ref{backward_time}.
Leaving aside the numerical issues discussed above, implicit differentiation costs $0.028s$,
which is roughly equal to the backward pass time of explicit differentiation for $10-15$ iterations.
This means that when the iterations needed in the forward pass is larger than $10-15$ iterations,
implicit differentiation may be preferable in terms of time for the backward pass, in
addition to potential memory savings an unrolled scheme would need for a large number of
iterations. 
\begin{table*}[!ht]

\centering
\resizebox{2.1\columnwidth}{!}{
\begin{tabular}{l  c c c c}
\hline
 &  \multicolumn{2}{c}{ CIFAR-FS 5-way} & \multicolumn{2}{c}{ FC100 5-way}\\
  LP Solver                        & 1-shot     & 5-shot     & 1-shot & 5-shot\\
\hline
MAML  \cite{finn2017model}                            & $58.9 \pm 1.9 $  &  $71.5 \pm 1.0 $  &  $- $              & $- $            \\
Prototypical Networks  \cite{snell2017prototypical}          & $55.5 \pm 0.7 $  &  $72.0 \pm 0.6 $  &  $35.3 \pm 0.6 $   & $48.6 \pm 0.6 $ \\
Relation Networks  \cite{sung2018learning}             & $55.0 \pm 1.0 $  &  $69.3 \pm 0.8 $  &  $- $              & $- $            \\
R2D2        \cite{bertinetto2018meta}                   & $65.3 \pm 0.2 $  &  $79.4 \pm 0.1 $  &  $- $              & $- $            \\
TADAM     \cite{oreshkin2018tadam}                        & $- $             &  $- $             &  $40.1 \pm 0.4 $   & $56.1 \pm 0.4 $ \\
ProtoNets(with backbone in \cite{lee2019meta})  \\
  \cite{snell2017prototypical}                      & $72.2 \pm 0.7 $  &  $83.5 \pm 0.5 $  &  $37.5 \pm 0.6 $   & $52.5 \pm 0.6 $ \\

MetaOptNet-RR \cite{lee2019meta}               & $72.6 \pm 0.7 $  &  $84.3 \pm 0.5 $  &  $40.5 \pm 0.6 $   & $55.3 \pm 0.6$  \\
MetaOptNet-SVM \cite{lee2019meta}             & $72.0 \pm 0.7 $  &  $84.2 \pm 0.5 $  &  $41.1 \pm 0.6 $   & $55.5 \pm 0.6$  \\

MetaOptNet-CVXPY-SCS             & $70.2\pm 0.7 $  &  $83.6 \pm 0.5 $  &  $ 38.1 \pm 0.6 $   & $51.7 \pm 0.6$  \\
MetaOptNet-Optnet (with regularization) & $69.9\pm 0.7$ & $83.9\pm0.5$ & $37.3\pm 0.5$ & $52.2\pm 0.5$ \\
MetaOptNet-$\gamma-$AuxPD (Ours)             & $71.4 \pm 0.7 $  &  $84.3 \pm 0.5 $  &  $38.2 \pm 0.5 $   & $54.2 \pm 0.5$  \\
\hline$ $
\end{tabular}
}
\caption{More baseline results on CIFAR-FS and FC100. We achieve comparable performance using $\ell_1$-SVM with  \cite{lee2019meta} which uses $\ell_2$-SVM and surpasses previous baseline methods. The choice between $\ell_1$ and $\ell_2$ often depends on the specific application considered, and $\ell_1$ is often faster to solve than $\ell_2$. Using the same $\ell_1$-SVM, our solver achieves better performance than CVXPY-SCS and  Optnet while being faster in terms of training time.}
\label{table_cifar_app}
\end{table*}

\section{Conclusions}
This paper describes 
how Physarum dynamics based ideas \cite{straszak2015natural,johannson2012slime} can 
be used to obtain a differentiable LP solver that 
can be easily integrated within 
various deep neural networks if the task involves obtaining a 
solution to a LP. 
Outside of the tasks shown in our experiments, there are many
other use cases including differentiable isotonic regression for calibration,
differentiable calculation of Wasserstein Distance, differentiable tracking, and so on. 
The algorithm,  $\gamma-$AuxPD, 
converges quickly without requiring 
a feasible solution as an initialization, and 
is easy to implement/integrate. 
Experiments demonstrate that 
when we preserve existing pipelines for 
video object segmentation and separately for meta-learning for few-shot learning, 
with substituting in our simple $\gamma-$AuxPD layer, we obtain comparable performance as more specialized schemes.  
  As briefly discussed earlier, recent results that utilize
  implicit differentiation to solve combinatorial problems  \cite{vlastelica2019differentiation}
  or allow using blackbox solvers for an optimization problem during DNN training \cite{berthet2020learning,ferber2020mipaal}, 
  are indeed promising developments because any state of the art solver can be utilized. 
  However, current LP solvers are often implemented to be CPU-intensive and
  suffer from overhead compared with solvers that are
  entirely implemented on the GPU. This is beneficial for DNN training.
Our code is available at \text{https://github.com/zihangm/Physarum-Differentiable-LP-Layer} and integration with CVXPY is ongoing,
which will complement functionality offered by tools like OptNet and CVXPY-SCS.

\section{Acknowledgements}
     We would like to thank one
          of the anonymous AAAI 2021 reviewers who apart from 
          suggestions also provided an 
          alternative implementation that improved the performance of
          CVXPY-SCS in our experiments. This helped strengthen our
          evaluations.  
          We thank Damian Straszak and Nisheeth Vishnoi for helpful clarifications
          regarding the convergence of continuous time physarum dynamics, and Yingxin Jia
          for interfacing our solver with a feature matching 
          problem studied in computer vision (https://github.com/HeatherJiaZG/SuperGlue-pytorch).
          This research was supported
          in part by UW CPCP AI117924, NSF CCF \#1918211, NIH R01 AG062336
          and R01 AG059312, NSF CAREER award RI\#1252725 
          and American Family Insurance.
          Sathya Ravi was also supported by UIC-ICR start-up funds.

\section{Appendix}

\subsection{Proof of Theorem 2 }\label{app:thm2_pf}
\begin{proof}
It is sufficient to show that $\gamma_u=\Theta(\sqrt{m+n})$. But showing such a constant exists is equivalent to showing that there is a neighborhood $\mathcal{N}=\mathcal{B}(c,r)$ around the cost vector or objective function $c$ of radius $r>0$ such that the optimal values of any two cost $c_1,c_2\in \mathcal{N}$ coincide i.e., there exists $x^*\in P$ such that $c_1^Tx^*=c_2^Tx^*$. To see that this is sufficient for our purposes, note that we can add small but positive constant to all the coordinates in $c$ that correspond to auxiliary/slack variables. Now, it is easy to see that Assumptions 1 and 2 guarantee that the optimal solution set is a {\em bounded} polyhedral multifunction.  Hence, we can use the Sticky Face lemma \cite{robinson2018short} to guarantee that such a nonzero $r$ exists. To conclude, we observe from the proof of the Sticky Face lemma,  that $r$ can be upper bounded by $1/M$, where $M$ corresponds to the the diameter of $P$ which is  $\Theta(\sqrt{m})$.
\end{proof}

\subsection{Proof of Convergence of $\ell_1$-SVM}\label{app:l1_svm}
Since the SVM formulation is always feasible, by the separating hyperplane theorem, there exists a $\kappa>0$ such that the when we add cost of $\kappa$ to each coordinate of $\alpha_1,\alpha_2,b_1,b_2,p,q,r$, then the (cost) perturbed linear program and the original LP ((6) in the main paper), have the same optimal solution. Then, it is easy to see that $C_s$ of this perturbed problem is quadratic in $n,C$ and $\kappa$. By scaling the data points, we can assume that \begin{align}
    \|x_i\|_2\leq 1.\label{eq:assmball}
\end{align} We now bound the magnitude of sub-determinant $D$ of the perturbed SVM LP. First note that the slack variables are diagonal, hence, the contribution to the determinant will be at most $1$. Hence, to bound $D$, we need to bound the determinant of the kernel matrix $K(X,X)$. Using Fischer's inequality \cite{thompson_1961}, we have that, \begin{align}
    D\leq \left(K\left(x_i,x_i\right) \right)^n.\label{eq:detsvm}
\end{align}
For a linear kernel, we have that, $D=\|x_i\|^n\leq 1$ (by assumption \eqref{eq:assmball}). For a Gaussian kernel scale $\sigma$, we have that, $D=O(\sigma)$ with high probability. We can easily extend this to any bounded kernel $K$.

More baseline results on the meta-learning experiments are shown in Table \ref{table_cifar_app}.

\clearpage
\bibliography{Main}
\end{document}